\documentclass{article}[11pt] 
\usepackage[vmargin=1.3in,hmargin=0.8in]{geometry}

\usepackage{amsmath,amssymb,amsfonts,graphicx,amsthm,nicefrac,mathtools,bm}
\usepackage{hyperref} 
\usepackage[numbers]{natbib}
\usepackage[english]{babel}
\usepackage{times}
\usepackage[T1]{fontenc}
\usepackage{bbm,mdframed}
\usepackage[dvipsnames]{xcolor}
\usepackage{xspace}
\usepackage{rotating,multirow} 
\usepackage{tikz}
\usetikzlibrary{arrows}
\usetikzlibrary{shapes} 
\usepackage{algpseudocode,algorithm,algorithmicx} 
\usepackage{fancyhdr} 
\usepackage{tabularx}
\usepackage{subcaption}
\usepackage{multirow}

\usepackage{float}
\usepackage{soul}
\usepackage{color}
\usepackage{enumitem}
\usepackage{comment}
\date{}
\author{\theauthor}
\title{\thetitle} 
\fancypagestyle{plain}{%
  \fancyhf{}%
  \lhead{\thepage}
}

\newcommand{\theauthor}{}
\newcommand{\thetitle}{LS-Tree: Model Interpretation When the Data Are Linguistic}

\newcommand{\powerset}{\raisebox{.15\baselineskip}{\Large\ensuremath{\wp}}}

\makeatletter
\long\def\@makecaption#1#2{
        \vskip 0.8ex
        \setbox\@tempboxa\hbox{\small {\bf #1:} #2}
        \parindent 1.5em  
        \dimen0=\hsize
        \advance\dimen0 by -3em
        \ifdim \wd\@tempboxa >\dimen0
                \hbox to \hsize{
                        \parindent 0em
                        \hfil 
                        \parbox{\dimen0}{\def\baselinestretch{0.96}\small
                                {\bf #1.} #2
                                } 
                        \hfil}
        \else \hbox to \hsize{\hfil \box\@tempboxa \hfil}
        \fi
        }
\makeatother

\makeatletter
\newcommand\footnoteref[1]{\protected@xdef\@thefnmark{\ref{#1}}\@footnotemark}
\makeatother

\newtheorem{theorem}{Theorem}

\makeatletter \renewcommand*{\@fnsymbol}[1]{\ensuremath{\ifcase#1\or
    \dagger\or \dagger\or \ddagger\or \mathsection\or
    \mathparagraph\or \|\or **\or \dagger\dagger \or \ddagger\ddagger
    \else\@ctrerr\fi}} \makeatother

\begin{document}

\author{ Jianbo Chen  \hspace{3mm} Michael I. Jordan} 

    \date{University of California, Berkeley } 

  \maketitle







\begin{abstract}
We study the problem of interpreting trained classification models in the setting of linguistic data sets. Leveraging a parse tree, we propose to assign least-squares based importance scores to each word of an instance by exploiting syntactic constituency structure. We establish an axiomatic characterization of these importance scores by relating them to the Banzhaf value in coalitional game theory. Based on these importance scores, we develop a principled method for detecting and quantifying interactions between words in a sentence. We demonstrate that the proposed method can aid in interpretability and diagnostics for several widely-used language models.
\end{abstract}


\setlength{\abovedisplayskip}{3pt}
\setlength{\abovedisplayshortskip}{1pt}
\setlength{\belowdisplayskip}{3pt}
\setlength{\belowdisplayshortskip}{1pt}
\setlength{\jot}{3pt}
\setlength{\textfloatsep}{3pt}

\section{Introduction}
Modern machine learning models can be difficult to probe and understand after they have been trained.  This is a major problem for the field, with consequences for trustworthiness, diagnostics, debugging, robustness and a range of other engineering and human interaction issues surrounding the deployment of a model.

There have been several lines of attack on this problem.  One involves changing the model design or training process so as to enhance interpretability. This can involve retreating to simpler models and/or incorporating strong regularizers that effectively simplify a complex model.  In both cases, however, there is a possible loss of prediction accuracy.  Models can also be changed in more sophisticated ways to enhance interpretability;
for example, attention-based methods have yielded deep models for vision and language tasks that improve interpretability at no loss to prediction accuracy~\cite{ba2014multiple,xu2015show,gregor2015draw,chen2015abc, yang2016stacked,xu2016ask,vaswani2017attention}. 

Another approach treats interpretability as a separate problem from prediction. Given a predictive model, an interpretation method yields, for each instance to which the model is applied, a vector of importance scores associated with the underlying features.  Within this general framework, methods can be classified as being model-agnostic or model-aware. Model-aware methods require additional assumptions, or are specific to a certain class of models \cite{simonyan2013deep,bach2015pixel, shrikumar2016not,karpathy2016visualizing,sundararajan2017axiomatic, godin2018explaining}. Model-agnostic methods can be applied in a black-box manner to arbitrary models \cite{ribeiro2016should, baehrens2010explain,lundberg2017unified,vstrumbelj2010efficient,datta2016algorithmic, li2016understanding}. 

While the generality of the stand-alone approach to interpretation is appealing, current methods provide little opportunity to leverage prior knowledge about what constitutes a satisfying interpretation in a given domain. Such methods have been studied most notably in the setting of natural-language processing (NLP), where there is an ongoing effort to incorporate linguistic structure (syntactic, semantic and pragmatic) in machine learning models. Such structure can be brought to bear in the model, the interpretation of a model, or both.  For example, \citet{socher2013recursive} introduced a recursive deep model to understand and leverage compositionality in tasks such as sentiment detection. \citet{lei2016rationalizing} proposed to use a combination of two modular components, generator and encoder, to explicitly generate rationales and make prediction for NLP tasks. 

Compositionality, or the rules used to construct a sentence from its constituent expressions, is an important property of natural language. While current interpretation methods fall short of quantifying compositionality directly, there has been a growing interest in investigating the manner in which existing deep models capture the interactions between constituent expressions that are critical for successful prediction~\cite{li2015visualizing, lei2016rationalizing, li2016understanding,godin2018explaining}. However, existing approaches often lack a systematic treatment in quantifying interactions, and the generality to be applied to arbitrary models.  

In the current paper, we focus on the model-agnostic interpretation of NLP models. Our approach quantifies the importance of words by leveraging the syntactic structure of linguistic data, as represented by constituency-based parse trees. In particular, we develop the \textit{LS-Tree value}, a procedure that provides instance-wise importance scores for a model by minimizing the sum-of-squared residuals at every node of a parse tree for the sentence in consideration. We provide theoretical support for this  by relating it to the Banzhaf value in coalitional game theory \cite{banzhaf1964weighted}. 

Our framework also provides a seedbed for studying compositionality in natural language. Based on the LS-Tree value, we develop a novel method for quantifying interactions between sibling nodes on a parse tree captured by the target model, by exploiting Cook's distance in linear regression \cite{cook1977detection}. We show that the proposed algorithm can be used to analyze several aspects of widely-used NLP models, including nonlinearity, the ability to capture adversative relations, and overfitting.  In particular, we carry out a series of experiments studying four models---a linear model with Bag-Of-Word features, a convolutional neural network \cite{kim2014convolutional}, an LSTM \cite{hochreiter1997long}, and the recently proposed BERT model \cite{devlin2018bert}.

\section{Least squares on parse trees}
For simplicity, we restrict ourselves to classification. Assume a model maps a sentence to a vector of class probabilities. We use $f$ to denote the function that maps an input sentence $x=(x_1,\dots,x_d)$ to the log probability score of a selected class. Let $2^{[d]}$ denote the powerset of $[d]:=\{1,2,\dots,d\}$.  The parse tree maps the sentence to a collection of subsets, denoted as $\powerset\subset 2^{[d]}$, where each subset $S\in\powerset$ contains the indices of words corresponding to one node in the parse tree. See Figure~\ref{fig:example} for an example. By abuse of notation, we use $f(S)$ to denote the output of the model evaluated on the words with indices $S$, with the rest of the words replaced by zero paddings or some reference placeholder. We call $v:\powerset \to \mathbb R$ defined by $v(S) := f(S) - f(\emptyset)$ a \textit{characteristic function}, which captures the importance of each word subset to the prediction.
\begin{figure}[bt!]
\centering
\includegraphics[width=0.99\linewidth]{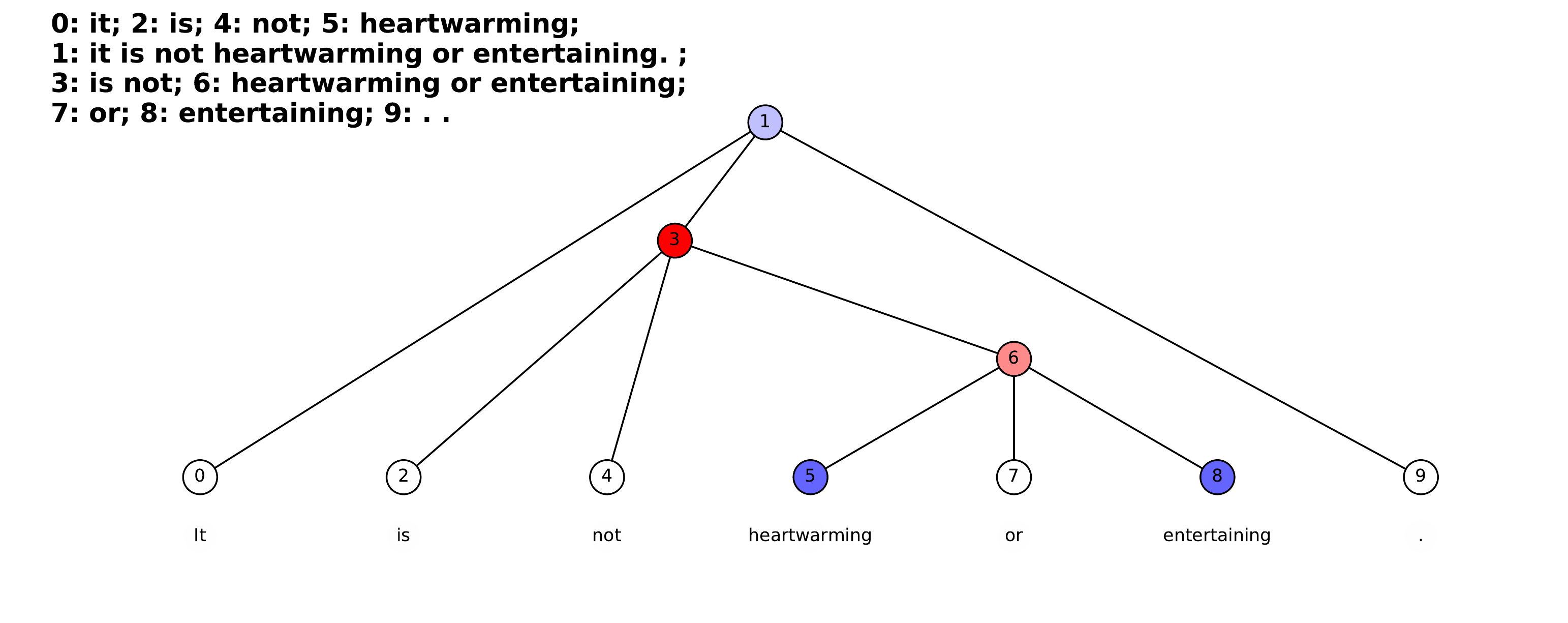} 
\caption{An example parse tree. Top left shows how each node corresponds to a word subset. Color indicates the direction and strength of interaction as assigned by Algorithm~\ref{alg:phrase}. Red is used for the direction of positive class, and blue otherwise.}
\label{fig:example} 
\end{figure} 
We seek the optimal linear function on the Boolean hypercube to approximate the characteristic function on $\powerset$, and use the coefficients as importance scores assigned to each word. Concretely, we solve the following least squares problem:
\begin{align}
  \underset{\psi\in \mathbb R^d}{\min} 
  \sum_{S\in\powerset}[v(S) - \sum_{i\in S} \psi_i]^2,\label{eq:ls}
\end{align}
where component $\psi_i$ of the optimal $\psi$ is the importance score of word with index $i$. We name the map from $(\powerset, v)$ to the solution to Equation~\eqref{eq:ls} the \textit{LS-Tree value}, because it results from least squares (LS) on parse trees, and can be considered as a \textit{value} in coalitional game theory.
\section{Connection to coalitional game theory} 
In this section, we give an interpretation of the LS-Tree value from the perspective of coalitional game theory. 

Model interpretation has been studied using tools from coalitional game theory \cite{vstrumbelj2010efficient,datta2016algorithmic,lundberg2017unified,chen2018lshapley}. We build on this line of research by considering a restriction on coalitions induced by the syntactic structure of the input.

Let $\powerset\subset 2^{[d]}$ be the collection of word subsets constructed from the parse tree. Taking each word as a player, we can define a coalitional game between $d$ words in a sentence as a pair $(\powerset, v)$, where $\powerset\subset 2^{[d]}$ enforces restrictions on coalition among players and $v: \powerset\to \mathbb R$ with $v(\emptyset) = 0$ is the characteristic function defined by the model evaluated on each coalition. A \textit{value} is a mapping that associates a $d$-dimensional payoff vector $\psi(\powerset, v)$ to each game $(\powerset, v)$, each entry corresponding to a word. The value provides rules which give allocations to each player for any game. 

The problem of defining a fair value in the setting of full coalition (when $\powerset = 2^{[d]}$) has been studied extensively in coalitional game theory \cite{shapley1953value,banzhaf1964weighted}. One popular value is the Banzhaf value introduced by \citet{banzhaf1964weighted}. For each $i\in [d]$ it defines the value:
\begin{equation*}
\phi_i(2^{[d]}, v) = \frac{1}{2^{d-1}}\sum_{S\subset N\setminus i}[v(S\cup i) - v(S)].
\end{equation*}
The Banzhaf value can be characterized as the unique value that satisfies the following four properties \cite{nowak1997axiomatization}: 

i) Symmetry: If $v(S\cup i) = v(S\cup j)$ for all $S\subset [d]\setminus \{i, j\}$, we have $\phi_i(2^{[d]}, v) = \phi_j(2^{[d]}, v)$.

ii) Dummy player property: If $v(S\cup i) = v(S) + v(i)$ for all $S\subset [d]\setminus i$, we have $\phi_i(2^{[d]}, v) = v(i)$. 

iii) Marginal contributions: For any two characteristic functions $v, w$ such that $v(S\cup i) - v(S) = w(S\cup i) - w(S)$ for any $S\subset [d]$, we have $\phi_i(2^{[d]}, v) = \phi_i(2^{[d]}, w)$.

iv) 2-Efficiency: If $i,j\in[d]$ merges into a new player $p$, then $\phi_p(2^{[d]\setminus \{i,j\}\cup p}, v^{ij}) = \phi_i(2^{[d]}, v) + \phi_j(2^{[d]}, v)$, where $v^{ij}(S) := v(S)$ if $p\notin S$ and $v^{ij}(S) := v(S\setminus p \cup i \cup j)$ otherwise, for any $S\subset [d]\setminus \{i,j\}\cup p$. 

These properties are natural for allocation of importance to prediction in model interpretation. Symmetry states that two features have the same allocation if their marginal contributions to feature subsets are the same. The dummy property states that a feature is allocated the same amount as the contribution of itself alone if its marginal contribution always equals the model evaluation on its own. The linear model yields such an example. Marginal contributions states that a feature which has the same marginal contribution between two models for any word subset has the same amount of allocation. 2-Efficiency states that allocation of importance is immune to artificial merging of two features.

To employ game-theoretic concepts such as the Banzhaf value in the interpretation of NLP models, we need to recognize that arbitrary combinations of words are not likely to be accepted as valid interpretations by humans. We might wish to start with a set of combinations that are likely to be interpretable by humans, and can be obtained via human-interpretable data, and then define the worth of other combinations of words via extrapolation. It turns out that the LS-Tree value as defined in the previous section can be interpreted as exactly such an extrapolation, where each node of the parse tree represents an interpretable word combination:
\begin{theorem}
Suppose a value $\psi$ coincides with the Banzhaf value $\phi$ for any game of full coalition, and for every game $(\powerset, v)$ with restricted coalition, it is consistent under the addition of an arbitrary subset $S\notin \powerset$:
\begin{equation}
\psi(\powerset, v) = \psi(\powerset \cup \{S\}, v'),
\end{equation}
where $v'$ is defined as $v'(T) = v(T)$ for $T\neq S$ and $v'(S) = \sum_{i\in S}\psi_i(\powerset, v)$. Then $\psi$ coincides with the LS-Tree value. 
\end{theorem}
\begin{proof}
It was shown in \citet{hammer1992approximations} that the Banzhaf value assigns to each player $i$ the corresponding coefficient in the best linear approximation of $v$. That is, 
\begin{align*}
  \phi(2^{[d]}, v) = \underset{\psi\in \mathbb R^d}{\arg\min} 
  \sum_{S\subset [d]}[v(S) - \sum_{i\in S} \psi_i]^2.
\end{align*} 
Following the proof of Theorem 3.3 in \citet{katsev2011least}\footnote{The original theorem is established for the solution to Problem~\eqref{eq:wls} with the efficiency constraint that $\sum_{i\in[d]}x_i = v([d])$. But the same proof follows for the unconstrained version.}, it follows directly that $\psi^*$, as is defined by Equation~\eqref{eq:wls}, is the unique value that coincides with $v\to\psi^*(2^{[d]}, v)$ with full coalition and is consistent under the addition of an arbitrary subset:
\begin{align}
  \psi^*(\powerset, v) = \underset{\psi\in \mathbb R^d}{\arg\min} 
  \sum_{S\in\powerset}w_{S}[v(S) - \sum_{i\in S} \psi_i]^2.\label{eq:wls} 
\end{align} 
Taking $w_S\equiv 1$, the theorem is established.
\end{proof}
\section{Detecting interactions}
We aim to detect and quantify interactions between words in a sentence that have been captured by the target model. While there are exponentially many possible interactions between arbitrary words, we restrict ourselves to the ones permitted by the structure of language. Concretely, we focus on interactions between siblings, or nodes with a common parent, in the parse tree. As an example, node $3$ in Figure~\ref{fig:example} represents interaction between ``is,'' ``not'' and ``heartwarming or entertaining.''

We define interaction as \textit{deviation of composition from linearity} in a given sentence. As a result, all non-leaf nodes in the tree are expected to admit zero interaction for a linear model. 
The above definition suggests that interaction can be quantified by studying how the inclusion of a common parent representing the interaction affects the coefficients of the linear approximation of the model.

Cook's distance is a classic metric in linear regression that captures the influence of a data point \cite{cook1977detection}. It is defined as a constant multiple of the squared distance between coefficients after a data point is moved, where the distance metric is defined by the data matrix $X\in\mathbb R^{n\times d}$:
\begin{equation*}
D_i = \text{Const.}\cdot (\hat\beta_{(i)} - \hat\beta)^TX^TX(\hat\beta_{(i)} - \hat\beta),
\end{equation*}
where $\hat \beta_{(i)}$ and $\hat \beta$ are the least squares estimate with the $i$th data point deleted and the original least squares estimate respectively. A larger Cook's distance indicates a larger influence of the corresponding data point. 

In our setting, the data matrix $X$ is a Boolean matrix where each row corresponds to a node in the tree, and an entry is one if and only if the word of the corresponding index lies in the subtree of the node. To capture the interaction of a non-leaf node $i$ (corresponding to some $S\in\powerset$), it does not suffice to only delete the corresponding row, because all of its ancestor nodes contain the segment represented by the node as well. To deal with this issue, we compute the distance between the least squares estimate with the rows corresponding to the node and all of its ancestors deleted, and the least squares estimate with only the rows corresponding to the ancestors deleted: 
\begin{equation}
D_i = d(\hat\beta_{(\geq i)}, \hat\beta_{(> i)})
\end{equation}
where $\hat\beta_{(\geq i)}, \hat\beta_{(> i)}$ denote the estimates with all its ancestors, including and excluding node $i$, deleted\footnote{Only entries corresponding to words within node $i$ will differ between $\hat\beta_{(\geq i)}$ and $\hat\beta_{(> i)}$, but we retain the remaining entries for notational simplicity.}. Cook's distance $d(a,b) = a^TX^TXb$ no longer has its statistical meaning here, as the normality assumption of the linear model no longer holds. A natural choice is the Euclidean distance $d_1(a,b):= \sqrt{a^Tb}$, which was also introduced by \citet{cook1977detection}. One drawback of the Euclidean distance is that it is unable to capture the direction of interaction. When this is an issue, we may use a signed distance: $d_2(a, b) := \sum_i (b_i - a_i)$, which sums up the influence of introducing the extra row on every coefficient of the linear model. We call the score defined by $d_1$ and $d_2$ absolute and signed \textit{LS-Tree interaction scores} respectively, as they are constructed from the LS-Tree value.

We propose an iterative algorithm to efficiently compute the interaction of each node on a tree with $n:=|\powerset|$ nodes. As a first step, $n$ model evaluations are performed, one evaluation for each node. 
For a node $i$, we denote as $\text{Ch}(i)$ the set of its children, $X_{(\geq i)}$ and $X_{(> i)}$ the data matrices excluding the ancestors of $i$, further excluding and including $i$ itself respectively, and $x_j^T$ the row corresponding to node $j$. The interaction score of each $j\in\text{Ch}(i)$ is a function of $\hat\beta_{(> j)} - \hat\beta_{(\geq j)}$. Denote $A_j = X_{(\geq j)}^TX_{(\geq j)}$. For each non-leaf node $j$, $A_j$ is of full rank and thus invertible. We show how $A_j^{-1}$ and $\hat\beta_{\geq j}$ can be computed from $A_i^{-1}$ and $\hat\beta_{\geq i}$. In fact, with an application of the Sherman-Morrison formula \cite{sherman1950adjustment}, we have 
\begin{align}
\hat \beta_{(>j)} &= (X_{(\geq j)}^TX_{(\geq j)} + x_j^Tx_j)^{-1}(X_{(\geq j)}^TY_{(\geq j)} + x_j^TY_j )\nonumber\\
&= (I - \frac{(X_{(\geq j)}^TX_{(\geq j)})^{-1}x_jx_j^T}{1 + x_j^T(X_{(\geq j)}^TX_{(\geq j)})^{-1}x_j})\hat\beta_{(\geq j)} \nonumber \\
& ~~~~~~~~~~~~~~+\frac{(X_{(\geq j)}^TX_{(\geq j)})^{-1}x_jY_j}{1 + x_j^T(X_{(\geq j)}^TX_{(\geq j)})^{-1}x_j}\nonumber\\
& = (I - \frac{A_j^{-1}x_jx_j^T}{1 + x_j^TA_j^{-1}x_j})\hat\beta_{(\geq j)} + \frac{A_j^{-1}x_jY_j}{1 + x_j^TA_j^{-1}x_j}.
\label{eq:pre-rearrange}
\end{align}
Rearranging the terms in Equation~\eqref{eq:pre-rearrange}, we have
\begin{align}
\hat \beta_{(\geq j)} &= \hat \beta_{(> j)} - A_j^{-1}x_j[Y_j - x_j^T\hat\beta_{(>j)}]. \label{eq:iterative2}
\end{align}
With another application of the Sherman-Morrison formula, we have 
 \begin{align}
A_j^{-1} &= (X_{(\geq i)}^TX_{(\geq i)} - x_jx_j^T)^{-1}\nonumber\\
&= A_i^{-1} + \frac{A_i^{-1}x_jx_j^TA_i^{-1}}{1 - x_j^T A_i^{-1}x_j}.
\label{eq:iterative}
\end{align}
For leaf nodes, the entry of $\hat\beta_{(\geq j)}$ corresponding to $j$ is set to zero, with the remaining entries equal to those of $\hat\beta_{(> j)}$. This is a result of the minimal Euclidean norm solution of Problem~\ref{eq:ls}, obtained from the pseudoinverse of $A_j$. Consequently, the (signed) interaction score of a leaf equals the model evaluation on the leaf alone.

We summarize the derivation in Algorithm~\ref{alg:phrase}, which traverses on the parse tree from root to leaves in a top-down fashion to compute the interaction scores of each node. As the number of nodes in a parse tree is linear in the number of words, Algorithm~\ref{alg:phrase} is of complexity $\mathcal O(d^3)$, plus the complexity of parsing the sentence, which is $\mathcal O(d)$ in our experiments, and $\mathcal O(d)$ model evaluations. Figure~\ref{fig:example} shows how Algorithm~\ref{alg:phrase} assigns signed interaction scores to a given example.
\begin{algorithm}[H]
       \caption{LS-Tree Interaction Detection}
       \label{alg:phrase}
       \begin{algorithmic}
       \Require Model $f$.
       \Require Sentence $x$.
       \Ensure LS-Tree value; interaction score.
          \Function {Detect-Interaction}{$f, x$}
          \State Find the parse tree $\mathcal T$ of $x$.
          \State Find the collection of subsets $\powerset$ corresponding to the parse tree.
           \For{each node $i$ in $\mathcal T$}
           \State Compute the model evaluation $v(S)$ for the corresponding subset $S$.
           \EndFor
           \State Compute LS-Tree value $\hat \beta$ for words via least squares.
           \State Find the root $r$ of $\mathcal T$.
           \State \textsc{Recursion}($v$, $\powerset$, $r$, $(X^TX)^{-1}$, $\hat \beta$)
           \EndFunction
       \end{algorithmic}
 \end{algorithm}
\vspace{-0.5cm}
\begin{algorithm}[H]
       \caption{Recursion}
       \label{alg:recursion}
       \begin{algorithmic}
       \Require $v, \powerset$, node $j$, $A^{-1}_{i}$, $\hat\beta_{(\geq i)}$

          \Function{Recursion}{$v, \powerset$, node $j$, $A^{-1}_{i}$, $\hat\beta_{(\geq i)}$}
          \If{$j$ is not a leaf}
            \State Compute $A_j^{-1}, \hat\beta_{(\geq j)}, D_j$ via Equation~\eqref{eq:iterative} and Equation~\eqref{eq:iterative2}.
            \For{each child $c$ in of $j$}
              \State Recursion($v$, $\powerset$, $c$, $A_j^{-1}$, $\hat\beta_{(\geq j)}$) 
            \EndFor 
           \Else
            \State Assign $D_j$ with $v(j)$ or $|v(j)|$. 
          \EndIf
          \EndFunction
       \end{algorithmic}
 \end{algorithm}
\section{Experiments}
We carry out experiments to analyze the performance of four different models: Bag of Words (BoW), Word-based Convolutional Neural Network (CNN) \cite{kim2014convolutional}, bidirectional Long Short-Term Memory network (LSTM) \cite{hochreiter1997long}, and Bidirectional Encoder Representations from Transformers (BERT) \cite{devlin2018bert}, across three sentiment data sets of different sizes: Stanford Sentiment Treebank (SST) \cite{socher2013recursive}, IMDB Movie reviews \cite{maas2011learning} and Yelp reviews Polarity \cite{zhang2015character}. For an instance with multiple sentences, we parse each sentence separately, and introduce an extra node as the common parent of all roots. Interactions between sentences are not considered in our experiments.

BoW fits a linear model on the Bag-of-Words features. Both CNN and LSTM use a 300-dimensional GloVe word embedding \cite{pennington2014glove}. The CNN is composed of three 100-dimensional convolutional 1D layers with 3, 4 and 5 kernels respectively, concatenated and fed into a max-pooling layer followed by a hidden dense layer. The LSTM uses a bidirectional LSTM layer with 128 units for each direction. BERT pre-trains a deep bidirectional Transformer \cite{vaswani2017attention} on a large corpus of text by jointly conditioning on both left and right context in all layers. It has achieved state-of-the-art performance on a large suite of sentence-level and token-level tasks. See Table~\ref{tab:data} for a summary of data sets and the accuracies of the four models.

We use the Stanford constituency parser \cite{goldberg2012dynamic,sagae2005classifier,zhang2009transition,zhu2013fast} for all the experiments. It is a transition-based parser that is faster than chart-based parsers yet achieves comparable accuracy, by employing a set of shift-reduce operations and making use of non-local features.

\subsection{Deviation from linearity}

We quantify the deviation of three nonlinear models from a linear model via the proposed LS-Tree value and interaction scores, both for specific instances and on a data set. 

The LS-Tree value can be interpreted as supplying the coefficients of the best linear model used to approximate the target model locally for each instance. The correlation between the LS-Tree value and the global linear model with Bag of Words (BoW) features can be used as a measure of nonlinearity of the target model at the instance. Table~\ref{tab:examples} shows three examples in SST, correctly classified by both BERT and BoW. BERT has low and high correlations with linear models at the first and second examples in Table~\ref{tab:examples} respectively. In particular, the top keywords, as ranked by the LS-Tree value, are different between two models. 

\begin{table}[bt!]
\centering
\resizebox{0.90\linewidth}{!}{
\begin{tabular}{|c|c|c|c|c|c|c|c|c|}
\hline\hline
Data Set & Classes & Train Samples & Test Samples & Avg. Length & BoW & CNN & LSTM & BERT \\
\hline\hline
SST \cite{socher2013recursive} & 2 & 6,920 & 872 & 19.7& 0.82\% & 0.85\% & 0.85\% & 0.93\%\\
\hline
IMDB \cite{maas2011learning} & 2 & 25,000 & 25,000 & 325.6& 0.94\% & 0.90\% & 0.88\% & 0.93\%\\ 
\hline
Yelp \cite{zhang2015character} & 2 & 560,000 & 38,000 & 136.2& 0.94\% & 0.95\% & 0.96\% & 0.97\% \\
\hline\hline
\end{tabular}
} 
\caption{Statistics of the three data sets, together with the test accuracy of the four models.}
\label{tab:data}
\end{table}

\begin{table}[t]
\centering
\resizebox{0.35\linewidth}{!}{
\begin{tabular}{|c|c|c|c|c|}
\hline
 & BoW & CNN & LSTM & BERT \\ 
\hline\hline 
SST &1.000 & 0.591& 0.580 & \textbf{0.465} \\
\hline
IMDB & 1.000 & 0.442 & 0.552 & \textbf{0.321} \\
\hline
Yelp &1.000 & 0.683 & 0.684& \textbf{0.476} \\
\hline
\end{tabular}
} 
\caption{Average correlation of LS-Tree values with the linear coefficients, comparable across different models on the same dataset.}
\label{tab:cor}
\end{table}

\begin{figure}[bt!]
\centering
\includegraphics[width=0.7\linewidth]{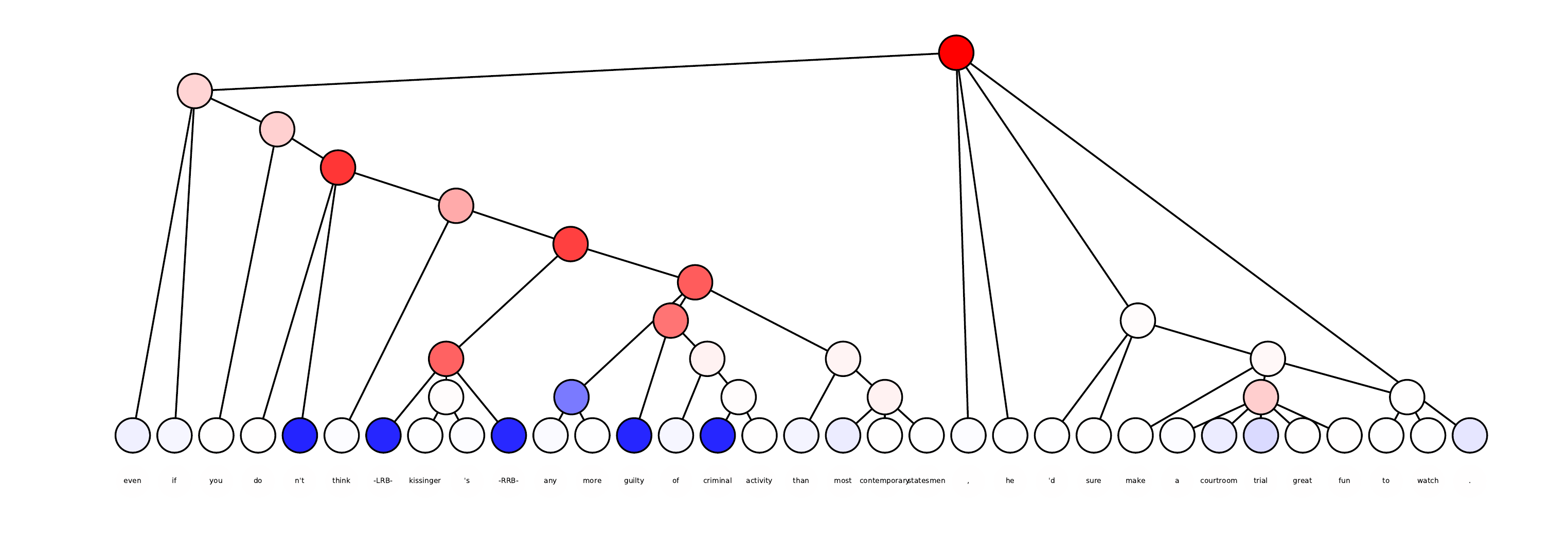} 
\includegraphics[width=0.7\linewidth]{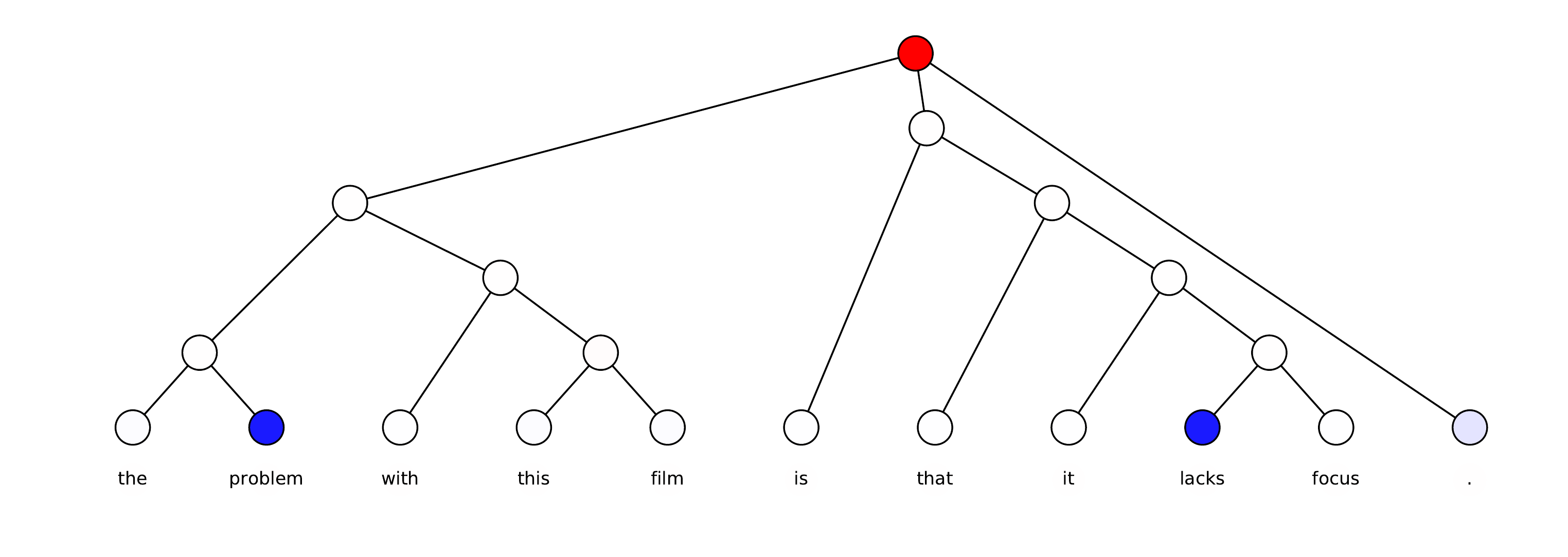}
\includegraphics[width=0.7\linewidth]{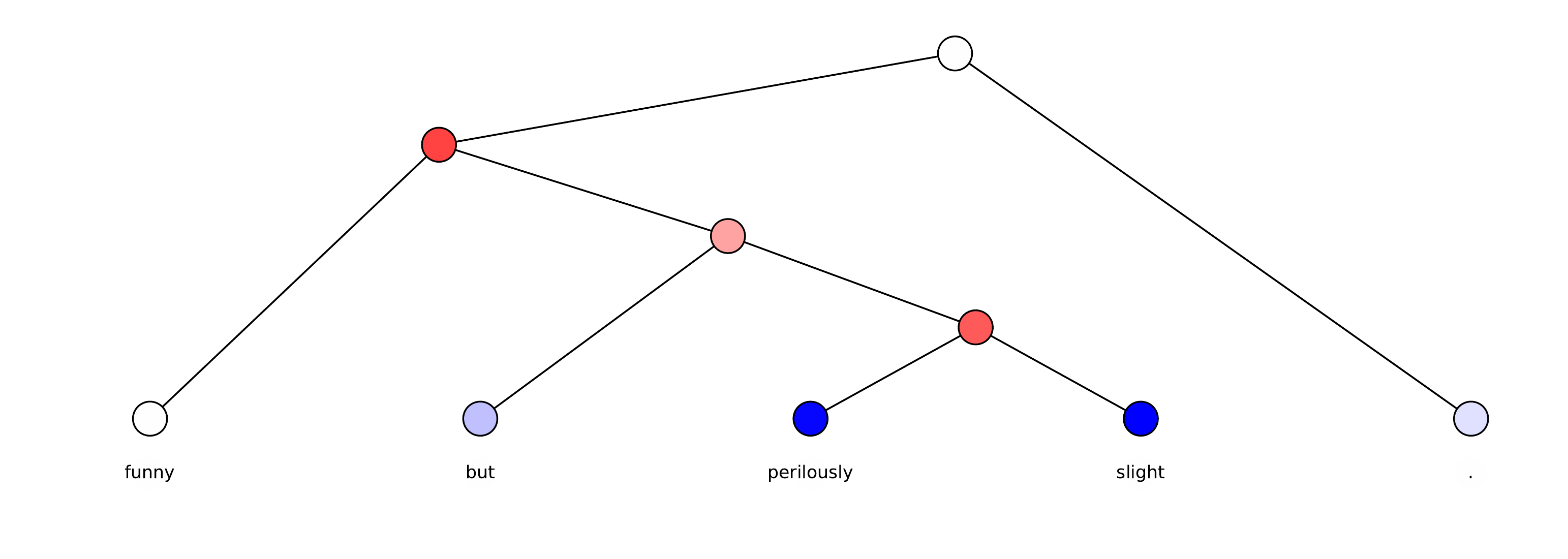} 
\caption{Visualization of parse trees of examples in Table~\ref{tab:examples}. Nodes are colorized based on the signed interaction scores, red for the direction of positive class, and blue otherwise.}
\label{fig:trees} 
\end{figure} 

The average of correlation with BoW across instances can be used as a measure of nonlinearity on a certain data set. The average correlation of BoW, CNN, LSTM and BERT with a linear model is shown in Table~\ref{tab:cor}, which indicates that BERT is the most nonlinear model among the four. CNN is more nonlinear than LSTM on IMDB but comparably nonlinear on SST and Yelp.

\begin{table}[t]
\centering
\resizebox{0.99\linewidth}{!}{
\begin{tabular}{|p{9cm}|p{9cm}|c|c|c|}
\hline\hline
BERT & BoW & Category & Correlation & Depth \\
\hline\hline
Even if you \colorbox[rgb]{1,0.0,0.0}{do} n't \colorbox[rgb]{1,0.5,0.5}{think} kissinger's any more guilty of criminal activity than most contemporary statesmen, he'd sure make a courtroom trial great fun to watch. &
Even if you don't think kissinger's any more guilty of criminal activity than most \colorbox[rgb]{1,0.5,0.5}{contemporary} statesmen, he'd sure make a courtroom trial great \colorbox[rgb]{1,0.0,0.0}{fun} to watch. & Positive & 0.173 & 11 \\

\hline
The \colorbox[rgb]{1,0.5,0.5}{problem} with this film is that it \colorbox[rgb]{1,0.0,0.0}{lacks} focus.&
The \colorbox[rgb]{1,0.5,0.5}{problem} with this film is that it \colorbox[rgb]{1,0.0,0.0}{lacks} focus.
& Negative & 0.939 & 1 \\
\hline

\colorbox[rgb]{1,0.0,0.0}{Funny} \colorbox[rgb]{1,0.5,0.5}{but} perilously slight. & 
\colorbox[rgb]{1,0.0,0.0}{Funny} \colorbox[rgb]{1,0.5,0.5}{but} perilously slight. & 
Positive & 0.938 & 4 \\
\hline\hline
\end{tabular}
} 
\caption{Examples from SST correctly classified by BERT and BoW. The correlation with the linear model, and Depth of the top node are listed. Top two words ranked by the LS-Tree value, and by the linear coefficients, are colorized.}
\label{tab:examples}
\end{table}

\begin{figure}[bt!]
\centering
\includegraphics[width=0.33\linewidth]{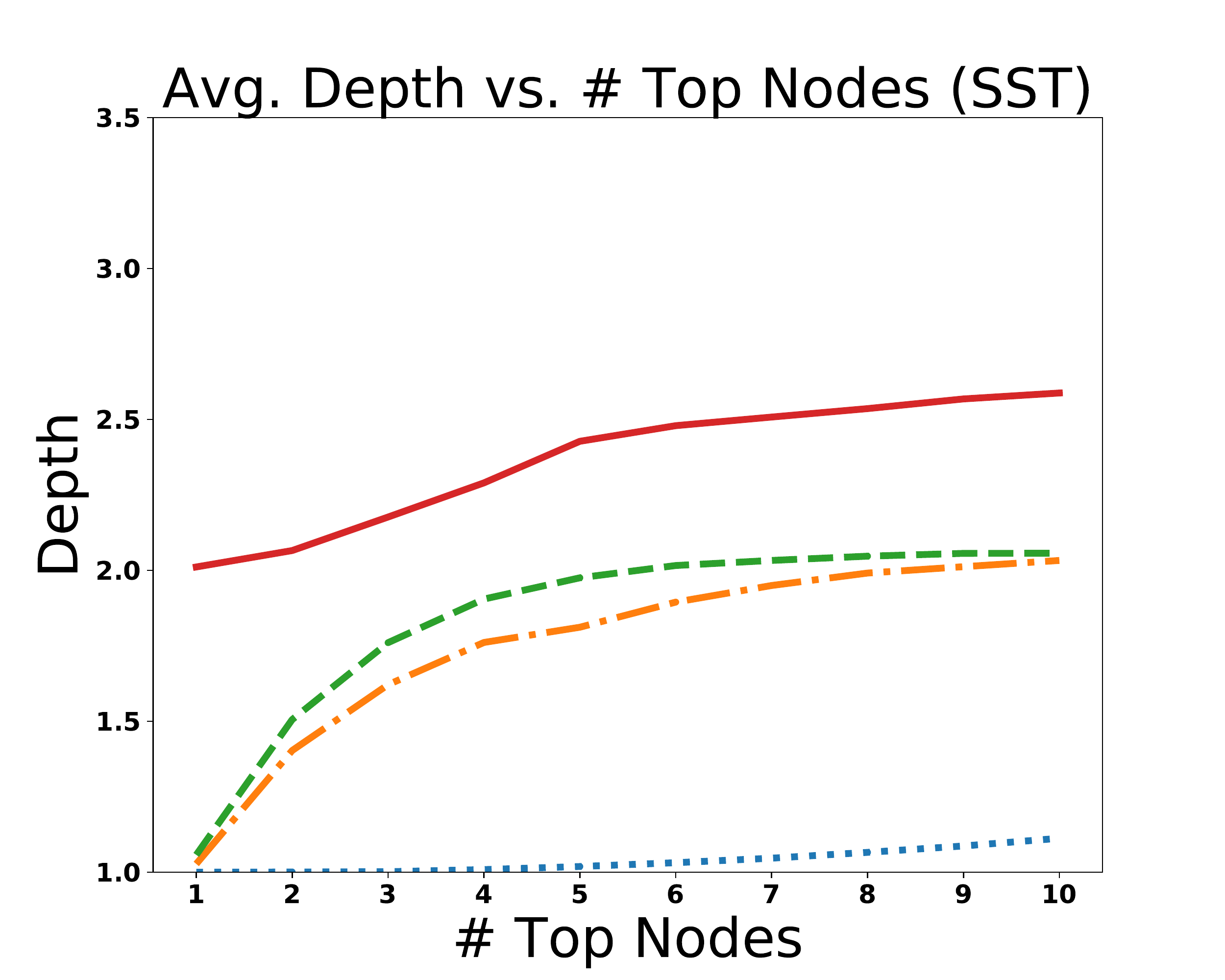} 
\includegraphics[width=0.33\linewidth]{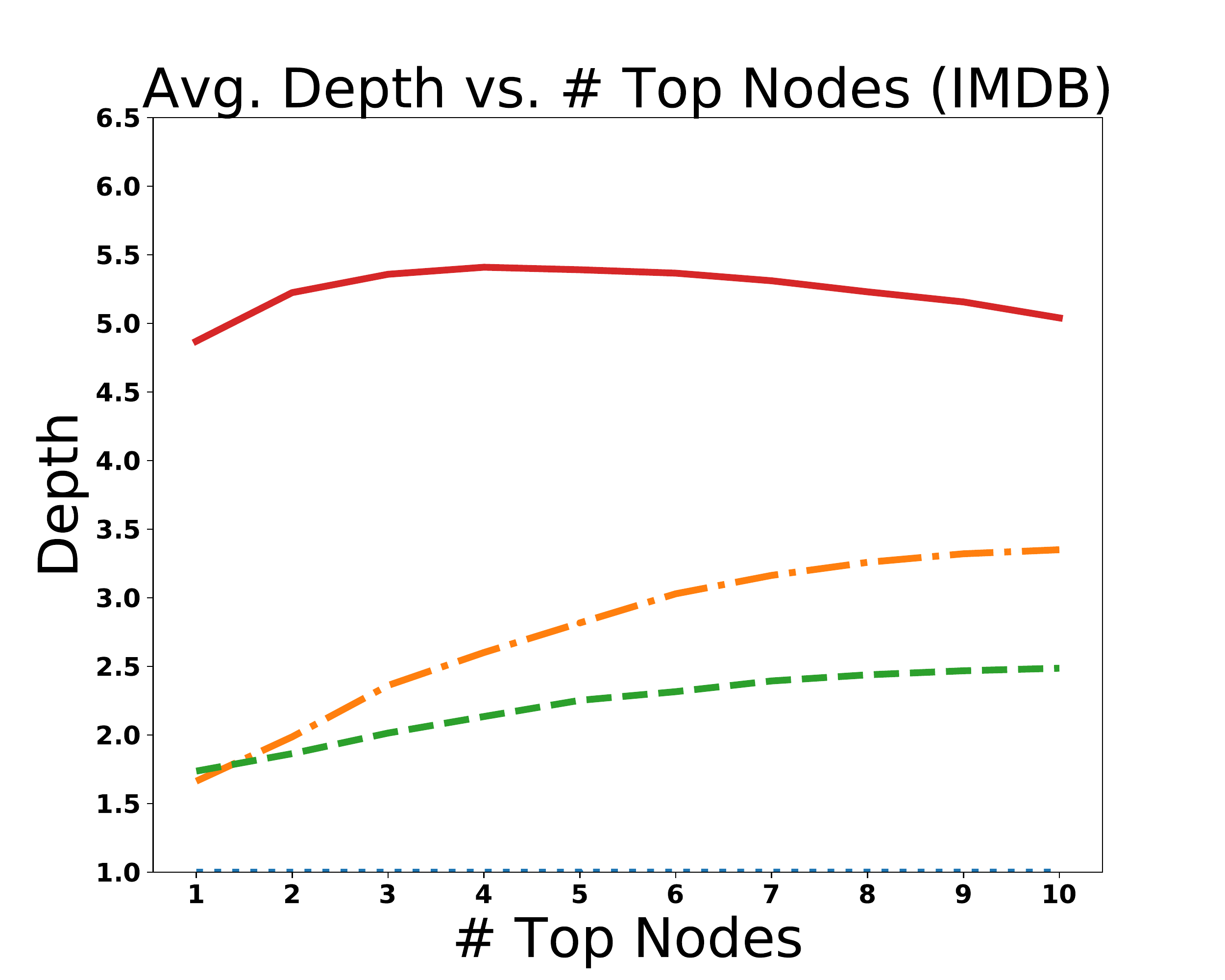}
\includegraphics[width=0.33\linewidth]{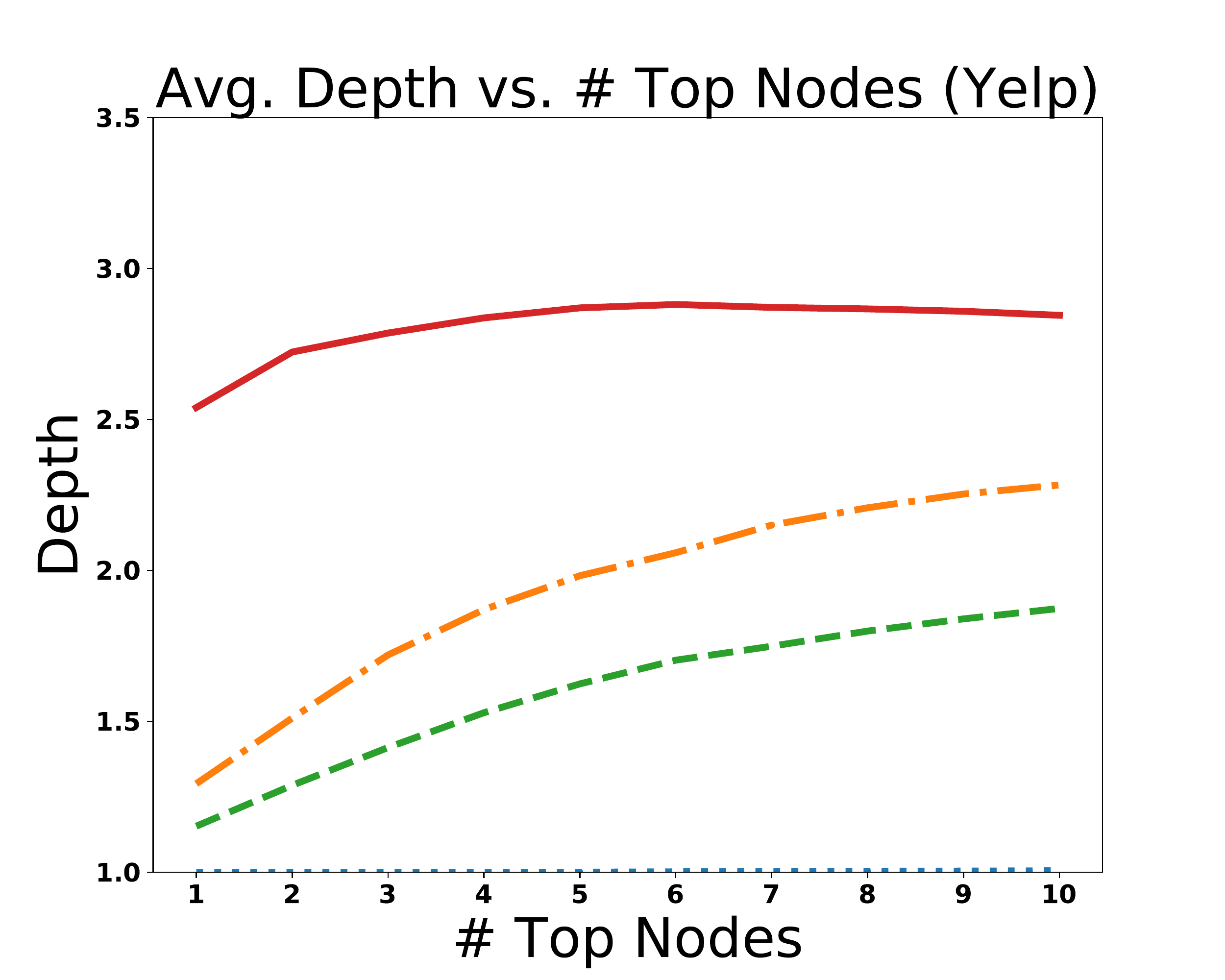}
\includegraphics[width=0.5\linewidth]{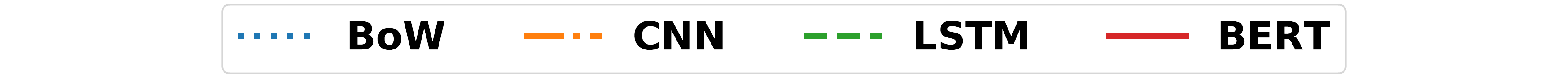}  
\caption{Average depth of top nodes as the number of the selected top nodes varies.}
\label{fig:depth} 
\end{figure}

Correlation alone may not suffice to capture the nonlinearity of a model. For example, the third sentence in Table~\ref{tab:examples} has a relatively high correlation, but the bottom left parse tree in Figure~\ref{fig:trees} indicates that the top interaction ranked by the signed interaction score is the node combining ``funny'' with ``but perilously slight''. This indicates the BERT model has captured the adversative conjunction, which BoW is not capable of. The ability to capture closer-to-the-top nodes in a parse tree is an indication of nonlinearity of the model.
To quantify this ability, we define the depth of a node in the parse tree as the maximum distance of the node from the bottom: 
\begin{align*}
\text{Depth}(i) = \begin{cases}
1 + \max_{c\in \text{Ch}(i)}\text{Depth}(c) &\text{if } \text{Ch}(i)\neq \emptyset,\\
1 &\text{otherwise.}
\end{cases}
\end{align*}
For a linear model, all non-leaf nodes have zero interaction, and thus the top ranked nodes are of depth 1, until all leaves with positive weights are enumerated. The higher the depth of top-ranked nodes, the more nonlinear a model is at a specific instance.

The average depths of top nodes ranked by interaction scores across instances can be used as a measure of the nonlinearity of the model on that data set. Figure~\ref{fig:depth} compares the average depths across BoW, CNN, LSTM and BERT on the three data sets, with top $k=1,2,\dots,10$ words selected. BoW is used as a baseline whose non-leaf nodes have zero interaction scores.  We use the absolute interaction scores here to capture all interactions, no matter in the same or opposite direction of prediction. BERT is still the most capable of capturing deeper interactions, followed by CNN and LSTM. CNN turns out to be a more nonlinear model than LSTM on Yelp, which was not captured by correlation.

\begin{table}[t]
\centering
\resizebox{1.0\linewidth}{!}{
\begin{tabular}{|c|c|c|c|c|c|c|c|c|c|c|c|c|}
\hline
\textbf{Dataset}&\textbf{Model} & \textbf{Avg. Score} & not &but &yet &though &although &even though &whereas & except &despite &in spite of \\
\hline\hline
\multirow{4}{*}{SST} & BoW & 0.153 &0.000(6.318) &0.000(0.079) &0.000(2.005) &0.000(0.865) &0.000(2.222) &0.000(0.000) &-(-) &0.000(4.280) &0.000(3.519) &0.000(0.000) \\
\cline{2-13}
& CNN & 0.634 &1.673(4.592) &1.694(1.444) &0.568(0.959) &0.213(0.735) &0.915(0.462) &0.626(0.407) &-(-) &0.948(1.175) &\textbf{1.452}(4.270) &\textbf{2.119}(1.943) \\
\cline{2-13}
& LSTM & 0.79 &\textbf{1.746}(2.580) &1.502(0.453) &1.449(2.368) &1.153(1.094) &0.338(0.197) &1.794(0.998) &-(-) &\textbf{2.353}(3.835) &1.256(1.818) &0.590(0.624) \\
\cline{2-13}
& BERT & 1.238 &1.714(4.383) &\textbf{2.148}(1.760) &\textbf{1.669}(3.120) &\textbf{1.525}(3.268) &\textbf{1.741}(3.256) &\textbf{1.885}(2.092) &-(-) &1.156(3.331) &1.160(2.998) &0.864(2.352) \\
\hline\hline
\multirow{4}{*}{IMDB} & BoW & 0.038 &0.000(2.683) &0.000(0.263) &0.000(2.210) &0.000(1.473) &0.000(1.710) &0.000(0.000) &0.000(3.604) &0.000(1.342) &0.000(0.132) &-(-) \\
\cline{2-13}
& CNN & 0.424 &1.050(0.819) &\textbf{3.442}(0.021) &\textbf{1.689}(0.295) &0.922(0.085) &1.036(0.071) &1.175(0.467) &0.469(1.064) &\textbf{1.590}(4.067) &0.363(0.434) &-(-) \\
\cline{2-13}
& LSTM & 0.126 &0.960(3.087) &2.222(0.524) &1.500(0.238) &0.611(0.087) &0.492(1.270) &0.944(0.683) &\textbf{1.222}(3.865) &1.294(4.008) &0.286(0.508) &-(-) \\
\cline{2-13}
& BERT & 1.159 &\textbf{1.616}(2.057) &3.390(1.800) &1.644(1.152) &\textbf{1.371}(2.061) &\textbf{1.735}(2.123) &\textbf{1.457}(1.557) &0.285(0.430) &1.421(2.060) &\textbf{1.518}(2.241) &-(-) \\
\hline\hline
\multirow{4}{*}{Yelp} & BoW & 0.035 &0.000(8.488) &0.000(1.015) &0.000(3.553) &0.000(1.664) &0.000(1.128) &0.000(0.000) &0.000(0.536) &0.000(0.367) &0.000(1.213) &-(-) \\
\cline{2-13}
& CNN & 0.161 &\textbf{2.287}(3.467) &\textbf{2.454}(0.932) &0.516(0.043) &0.988(0.435) &\textbf{0.889}(0.075) &0.789(0.621) &0.286(0.671) &\textbf{0.522}(2.529) &0.423(0.889) &-(-) \\
\cline{2-13}
& LSTM & 0.224 &2.173(5.950) &1.712(1.676) &\textbf{0.988}(2.065) &0.984(1.310) &0.706(1.194) &0.559(0.483) &\textbf{1.395}(1.793) &0.344(1.408) &0.514(1.153) &-(-) \\
\cline{2-13}
& BERT & 0.746 &1.384(2.106) &2.448(0.658) &0.781(0.184) &\textbf{1.336}(0.953) &0.596(0.615) &\textbf{1.019}(0.880) &0.095(0.162) &0.331(0.074) &\textbf{1.041}(0.414) &-(-) \\
\hline\hline
\end{tabular}
} 
\caption{Ratio of interaction scores of specific nodes to the average interaction score of a generic node. Ratios in parentheses are for nodes with adversative words alone.  Ratios without parentheses are for their parent nodes where the adversative relation takes place.}
\label{tab:adversative}
\end{table}

\begin{table}[t]
\centering
\resizebox{0.99\linewidth}{!}{
\begin{tabular}{|p{8cm}|c|c|c|c|c|c|}
\hline
Sentence & Meaning & BoW & CNN & LSTM & BERT \\
\hline\hline
\small{... He said he couldn't help. We had to walk \textbf{while} the snow blew in our faces. When we were almost there, we saw the shuttle pull out with the smoking shuttle driver in it, driving in the opposite direction, away from us. I can not believe how rude they were.} & during the time that & 0.000(0.338)& 0.781(0.300)& 1.761(0.839)& 0.062(0.092) \\
\hline
\small{... I ordered a cappuccino. It tasted like milk and no coffee. I was exceptionally disappointed. So \textbf{while} the place has a great reputation, even they can screw it up if they don t pay attention to detail, and at this level they should never screw it up. I had a better cup at Martys Market for crying out loud!} & whereas (indicating a contrast) & 0.000(0.338)& 1.142(0.300)& 2.155(0.839)& 2.167(0.092) \\
\hline
\small{Usually asking the server what is her favorite dish gets you a pretty good recommendation, but in this case, it was crap! The smoked brisket had that discoloration that happens to meat when it's been sitting out for a \textbf{while}. And it wasn't even tender!! Am I asking for too much?} & a period of time & 0.000(0.338)& 0.206(0.300) & 0.465(0.839)& 0.082 (0.092)\\
\hline\hline
\end{tabular}
} 
\caption{The word ``while'' in different contexts, together with the ratio of interaction scores of ``while'' nodes and their parents to the average score of a generic node. Ratios in parentheses are for nodes with ``while'' alone. Scores without parentheses are for their parents.}
\label{tab:while}
\end{table}

\subsection{Adversative relations}

Adversative words are those which express opposition or contrast. They often play an important role in determining the sentiment of an instance, by reversing the sentiment of a preceding or succeeding word, phrase or statement. We focus on four types of adversative words: negation that reverses the sentiment of a phrase or word (e.g., ``not''), adversative coordinating conjunctions that express opposition or contrast between two statements (e.g., ``but'' and ``yet''), subordinating conjunctions indicating adversative relationship (e.g., ``though,'' ``although,'' ``even though,'' and ``whereas''), prepositions that precede and govern nouns adversatively (e.g., ``except,'' ``despite'' and ``in spite of''). 

In most cases, adversative words only function if they interact with their preceding or succeeding companion. In order to verify whether models are able to capture the adversative relationship, we examine the LS-Tree interaction scores of the parent nodes of these words.

We extract all instances that contain any of the above adversative words. Then for each word in an instance, we compute the interaction score of the corresponding node with the word alone, and that of its parent node. A high interaction score on the node with the adversative word alone indicates the model inappropriately attributes to the word itself a negative or positive sentiment. A high interaction score on the parent node indicates the model captures the interaction of the adversative word with its preceding or succeeding component. To compare across different models, we further compute the average interaction score of a generic node across all instances, and report the ratio of average interaction scores of specific nodes to the average score of a generic node for respective models. 

Table~\ref{tab:adversative} reports the results on three data sets. We observe the ability of capturing adversative relation for different models varies across data sets. BERT takes the lead in capturing adversative relations on SST and IMDB, perhaps with the help of BERT's pre-training process on a large corpus, but CNN and LSTM catch up with BERT on Yelp, which has a larger amount of training data. On the other hand, all models assign a high score on nodes with adversative words alone. This perhaps results from the uneven distribution of adversative words like ``not'' among the positive and negative classes. An additional observation is that BERT has the highest score for a generic node on average across three data sets, indicating that BERT is the most sensitive to words and interactions on average.


Some words have different meanings in different contexts. It is interesting to investigate whether a model can distinguish the same word under different contexts. The word ``while'' is such an example. Table~\ref{tab:while} shows three Yelp reviews that include ``while.'' It can be observed that the scores of the parent nodes of ``while'' is higher than average when ``while'' contains an adversative meaning, but lower otherwise. This observation holds across CNN, LSTM and BERT, with the sharpest distinction on BERT.

\begin{figure}[bt!]
\centering
\includegraphics[width=0.5\linewidth]{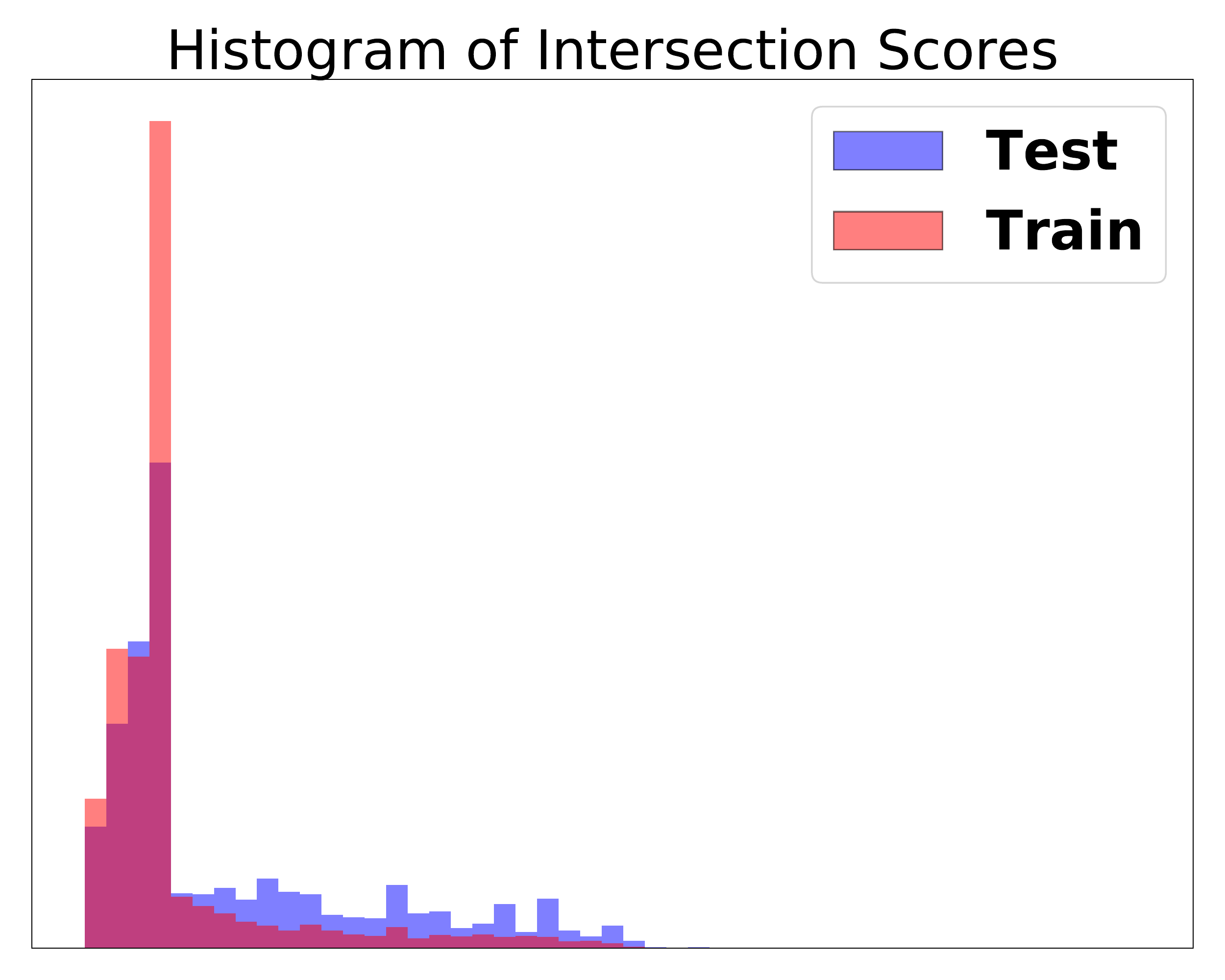}   
\caption{The histogram of intersection scores with train and test data for BERT on SST.}
\label{fig:hist} 
\end{figure}

\begin{figure}[bt!]
\centering
\includegraphics[width=0.33\linewidth]{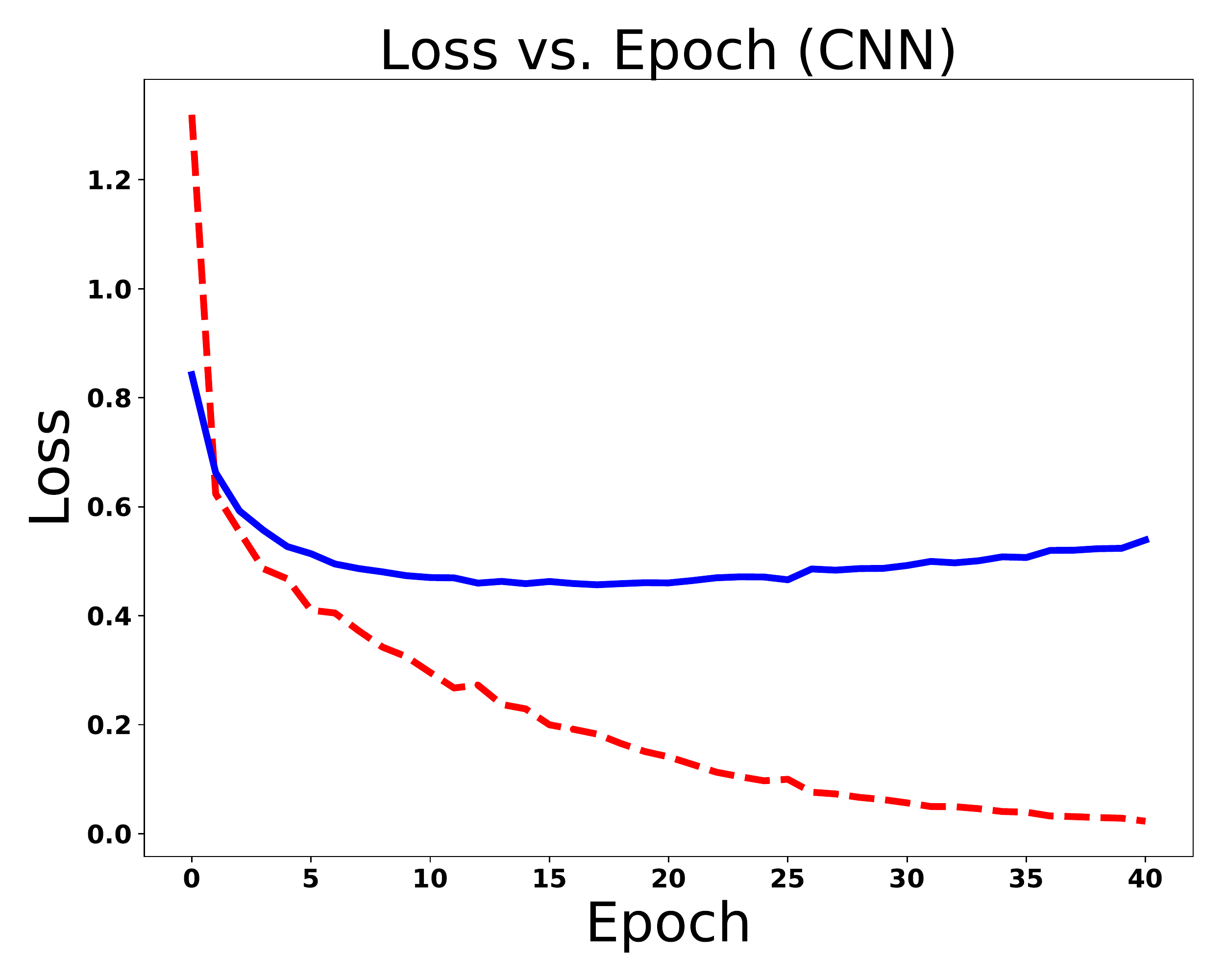} 
\includegraphics[width=0.33\linewidth]{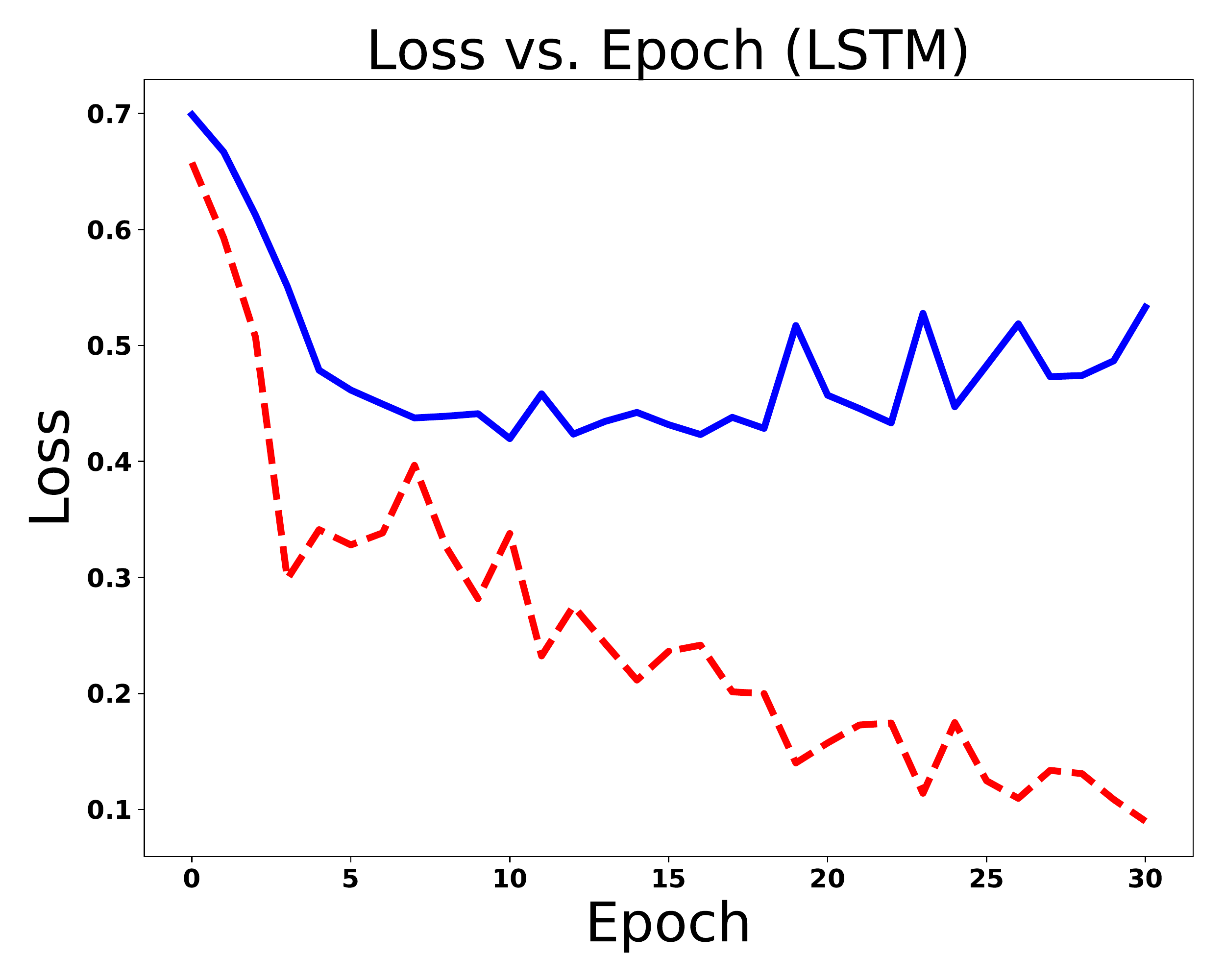} 
\includegraphics[width=0.33\linewidth]{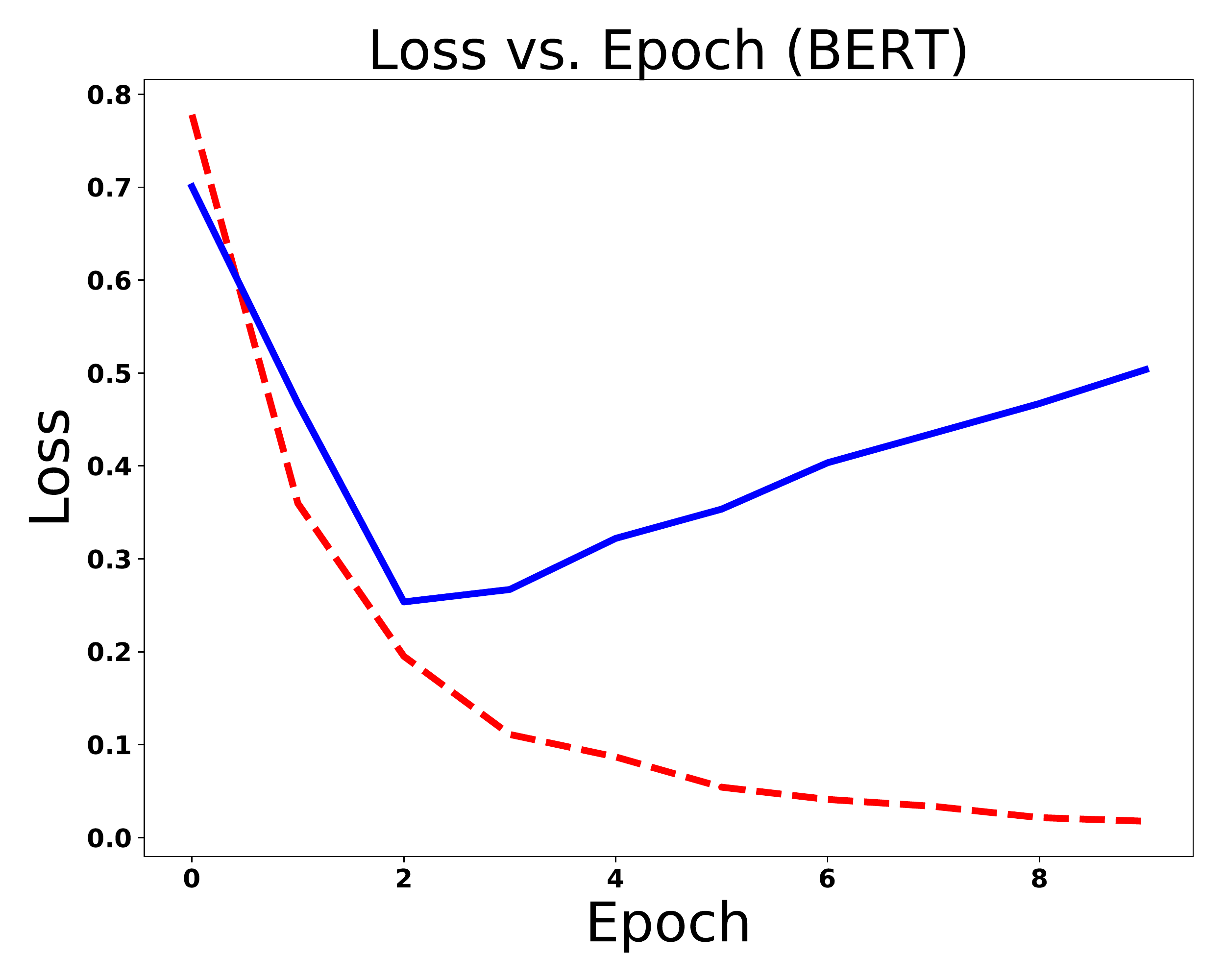} 
\includegraphics[width=0.33\linewidth]{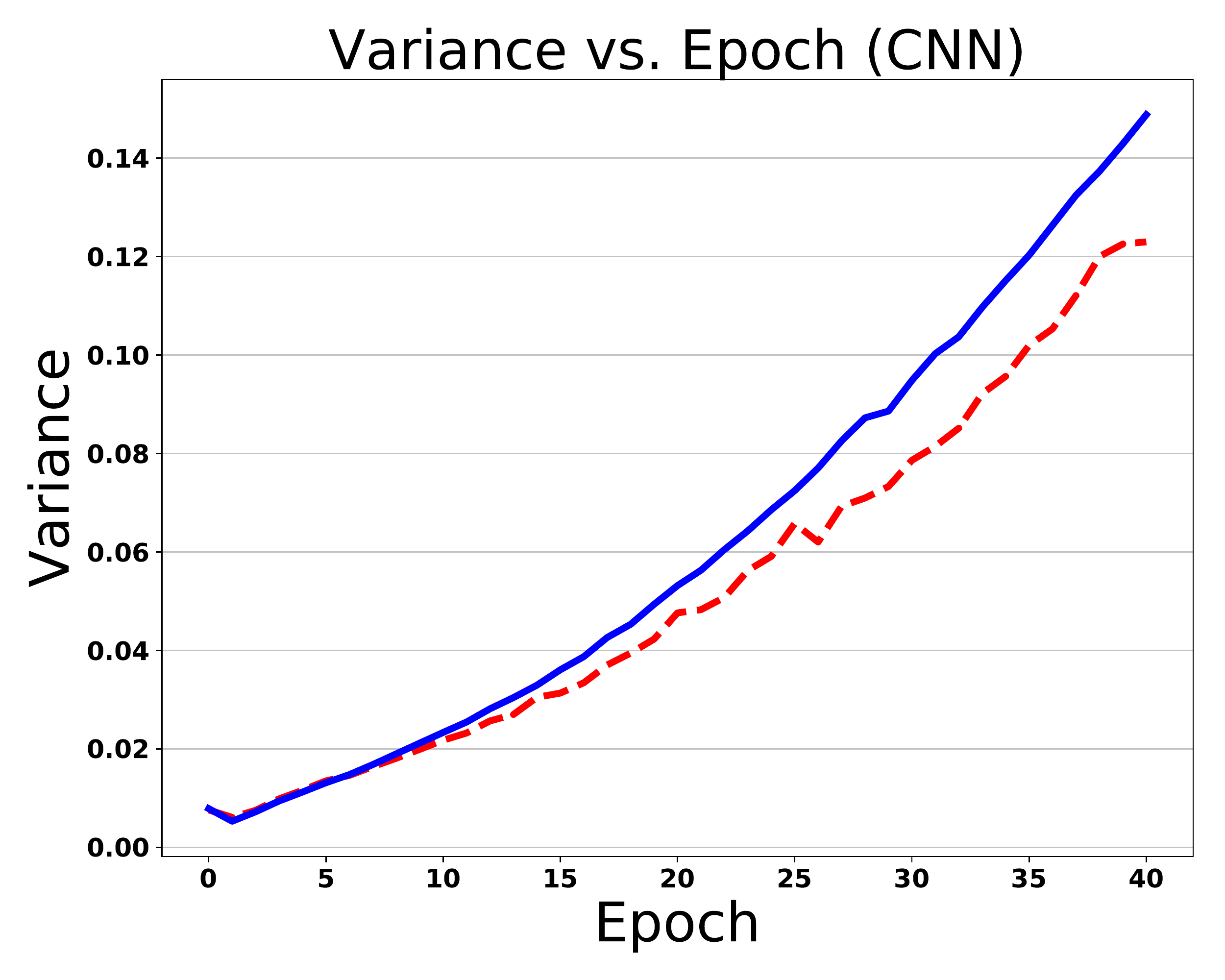} 
\includegraphics[width=0.33\linewidth]{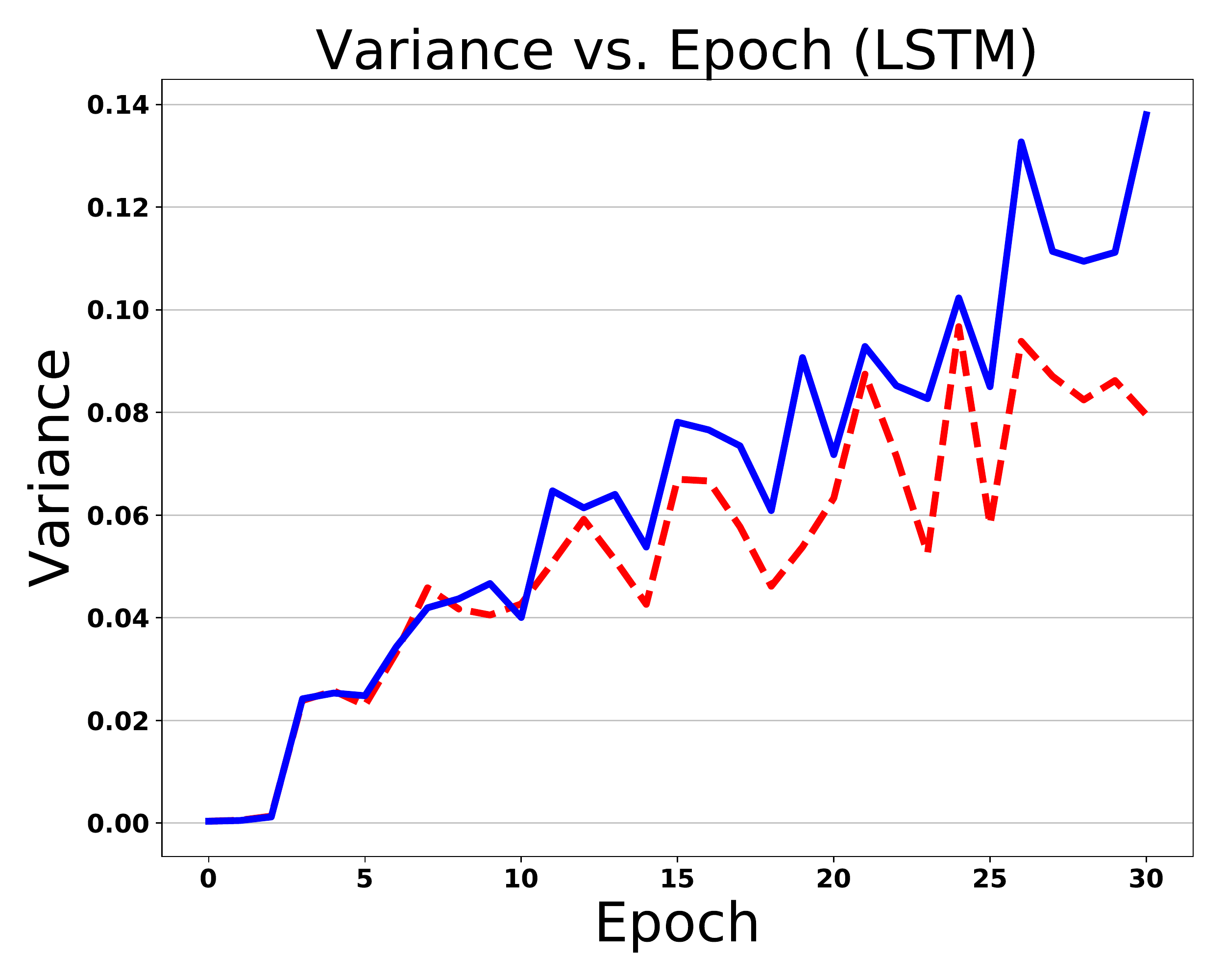} 
\includegraphics[width=0.33\linewidth]{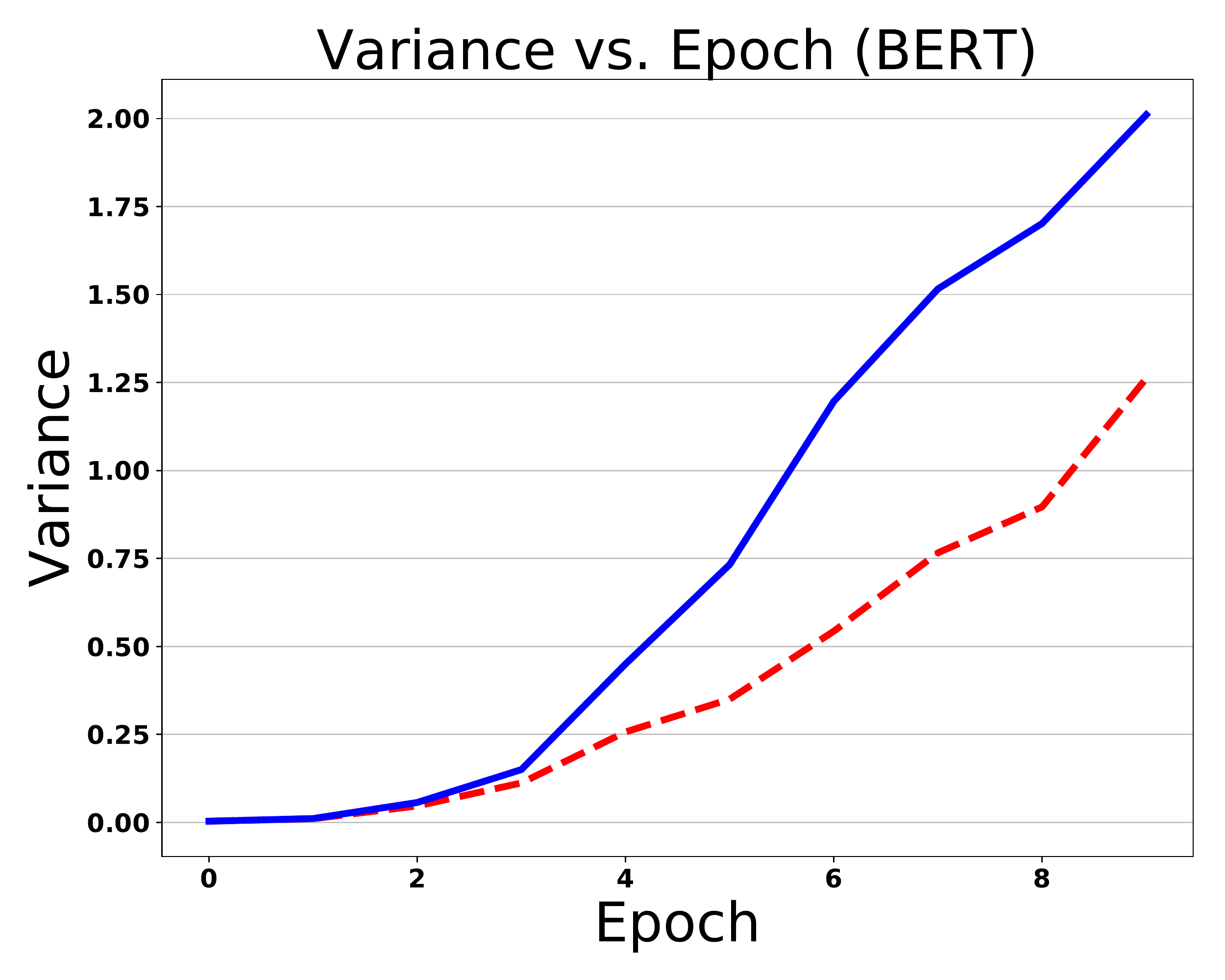} 
\includegraphics[width=0.33\linewidth]{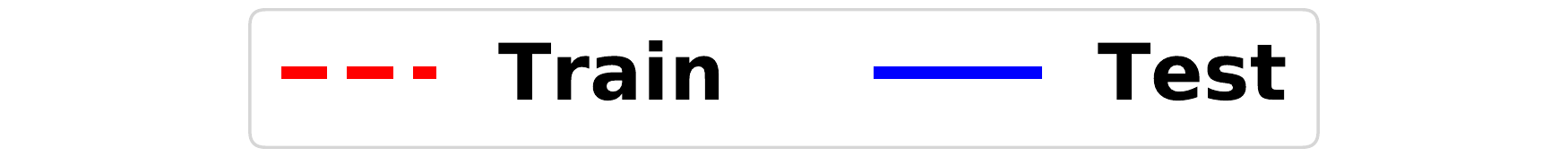} \\
\includegraphics[width=0.33\linewidth]{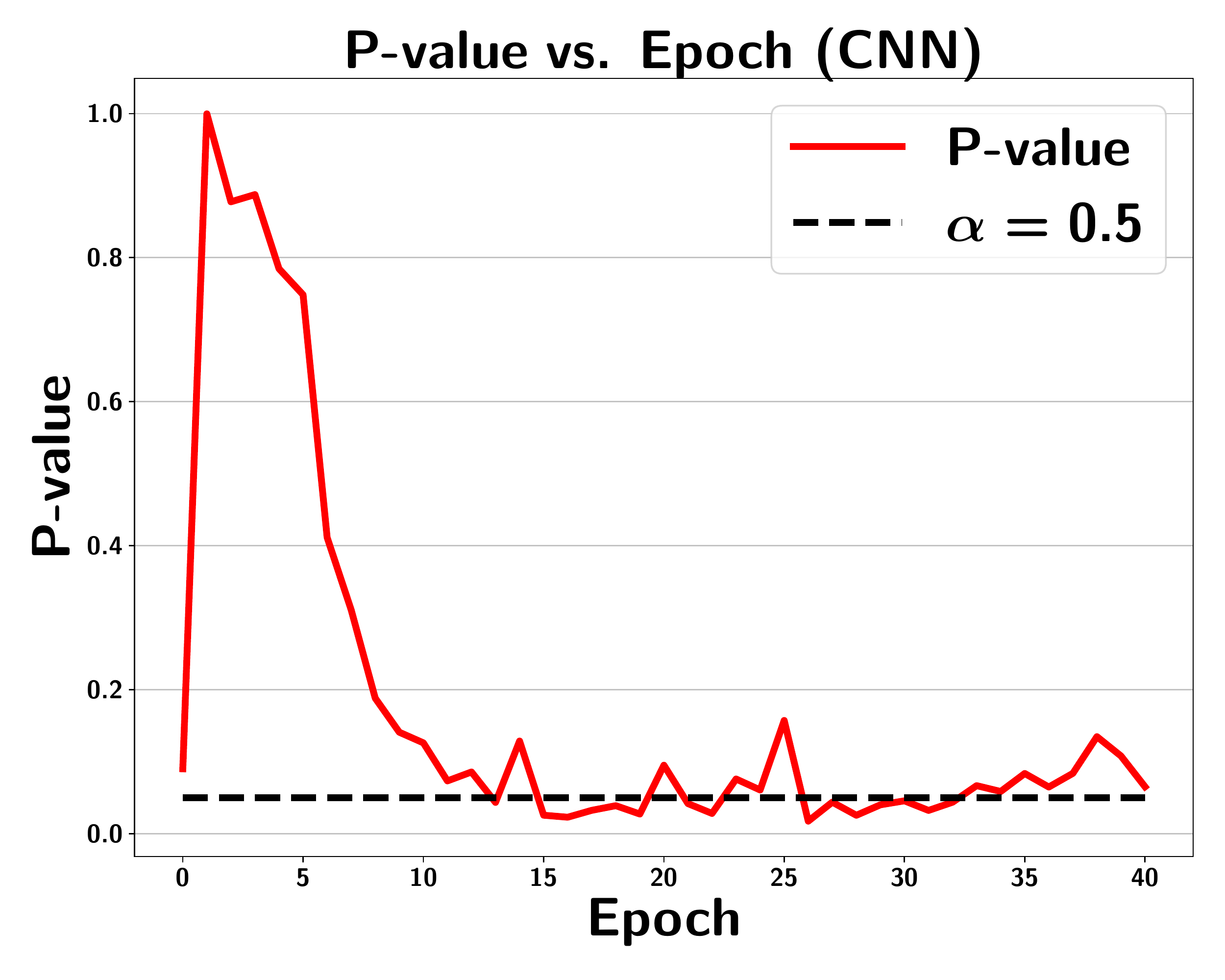} 
\includegraphics[width=0.33\linewidth]{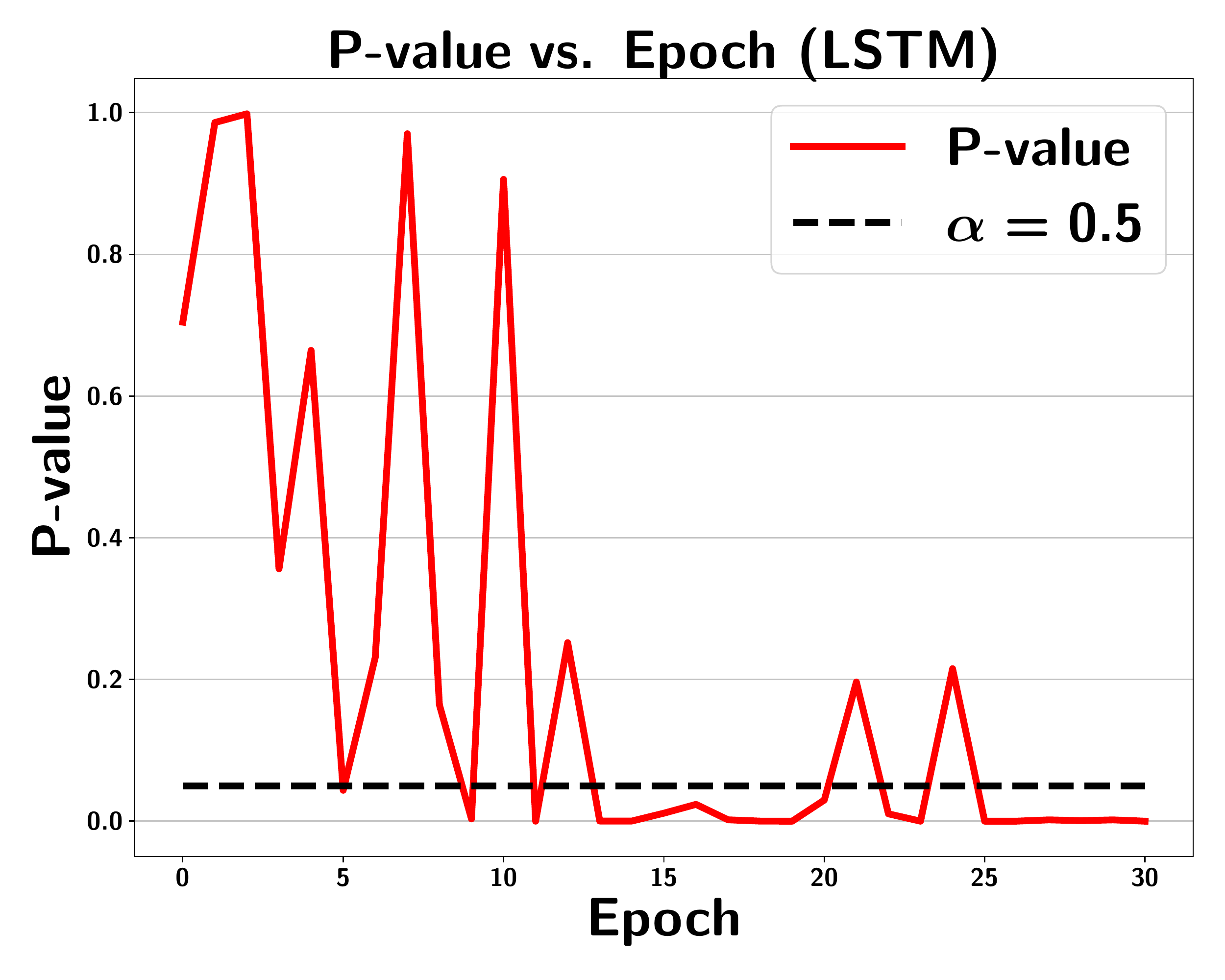} 
\includegraphics[width=0.33\linewidth]{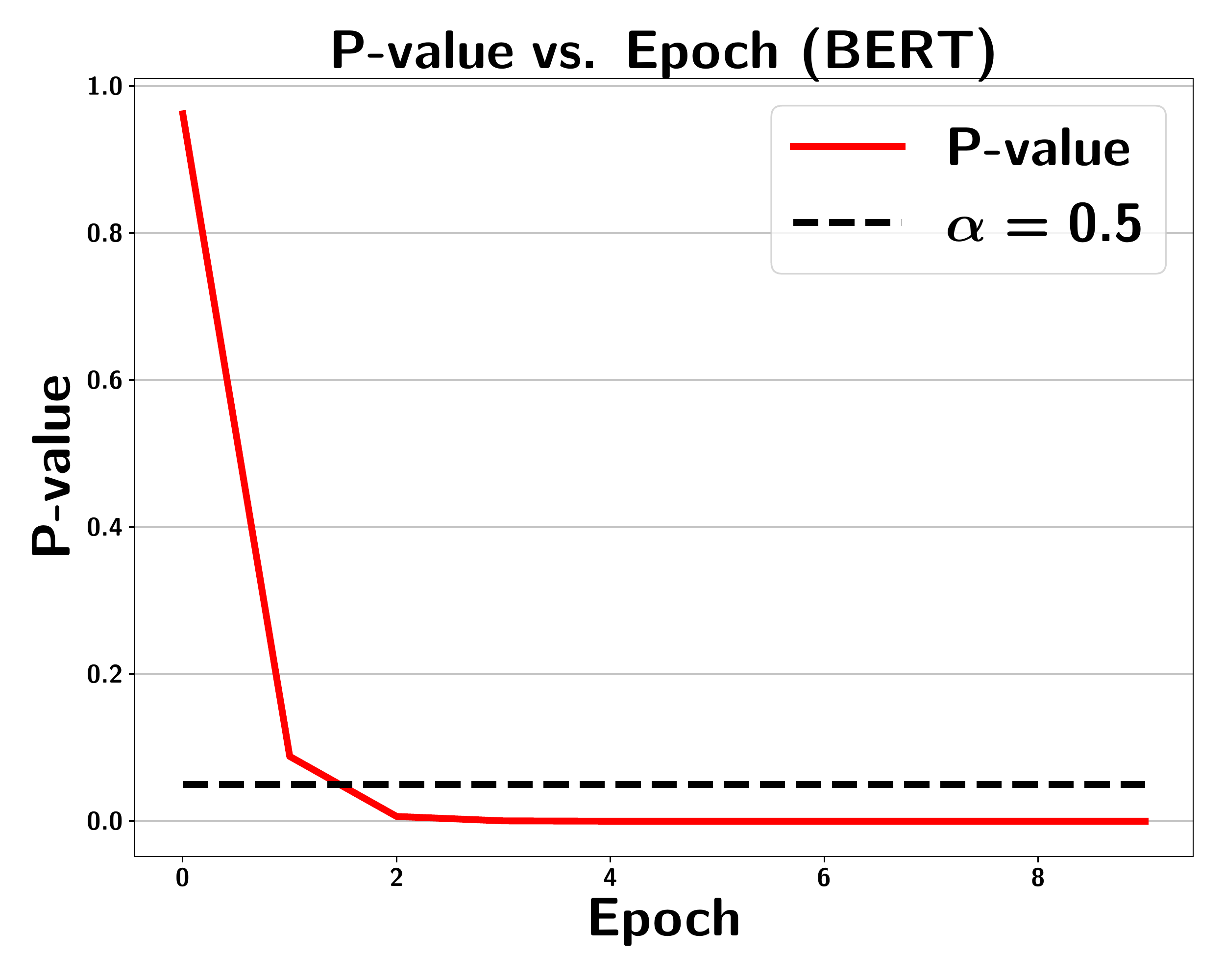} 

\caption{The three figures in Line 1 plot training and test loss of CNN, LSTM, BERT respectively. The figures in Line 2 plot the corresponding average variance of interaction scores across instances over training and test sets. The figures in Line 3 show p-values of permutation tests of $50,000$ iterations with $300$ randomly selected instances in training and test sets respectively.}
\label{fig:overfit} 
\end{figure}
\subsection{Detecting overfitting}
Overfitting happens when a model captures sampling noise in training data, while failing to capture underlying relationships between the inputs and outputs. Overfitting can be a problem in modern machine learning models like deep neural networks, due to their expressive nature. To mitigate overfitting, one often splits the initial training set into a training and a validation set, and uses the latter to obtain an estimate of the generalization performance \cite{larson1931shrinkage}. This leads to a waste of training data, depriving the model of potential opportunities to learn from the labelled validation data. We observe that the LS-Tree interaction scores can be used to construct a diagnostic for overfitting, one which is solely computed with unlabelled data. 

Figure~\ref{fig:hist} shows the histograms of absolute interaction scores on small subsets of training and test data of SST, for an overfitted BERT model. The scores are more spread out on test data than those on training data. In fact, we have observed this phenomenon holds true on average across instances for a overfitted model. 
In particular, interaction scores of test instances have a larger variance on average than those of training instances when the model is overfitted, but comparable otherwise. The observation can also be generalized to other types of neural networks, including CNN and LSTM. We show in Figure~\ref{fig:overfit} the average variance on training and test sets for CNN, LSTM and BERT models against training epochs, together with the loss curves. We observe that overfitting occurs when the variances between training and test sets differ.

The observation suggests we may use the difference of average variances of interaction scores between training and test sets as a diagnostic for overfitting. In particular, a permutation test can be carried out under the null hypothesis of equal average variance. The resulting p-values are plotted against the number of training epochs in the third line of Figure~\ref{fig:overfit}. It can be observed that p-values fall below the significance level of $0.05$ when overfitting occurs, which suggests the rejection of the null hypothesis as an early stopping criterion in training. 


\section{Discussion}

We have proposed the LS-Tree value as a fundamental quantity for interpreting NLP models.  This value leverages a constituency-based parser so that syntactic structure can play a role in determining  interpretations.  We have also presented an algorithm based on the LS-Tree value for detecting interactions between siblings of a parse tree. To the best of our knowledge, this is the first model-interpretation algorithm to quantify the interaction between words for arbitrary NLP models. We have applied the proposed algorithm to the problem of assessing the nonlinearity of common neural network models and the effect of adversative relations on the models. We have presented a permutation test based on the LS-Tree interaction scores as a diagnostic for overfitting.

{
\bibliography{ls_tree}
}

\bibliographystyle{plainnat}

\end{document}